\documentclass[letterpaper, 10 pt, conference]{ieeeconf}  %

\IEEEoverridecommandlockouts                              %

\usepackage{graphics} %
\usepackage{graphicx}
\usepackage{color}
\usepackage{times} %
\usepackage{amsmath} %
\usepackage{amssymb}  %
\usepackage{bm}
\usepackage{subcaption}

\usepackage{xcolor}
\usepackage{amssymb}
\usepackage{mathtools}
\usepackage{booktabs}
\usepackage{multirow}
\usepackage[hidelinks]{hyperref}
\usepackage[nameinlink]{cleveref}
\usepackage[T1]{fontenc}

\allowdisplaybreaks

\usepackage[backend=bibtex,style=ieee,natbib=true,,citestyle=numeric-comp]{biblatex} %

\usepackage{xcolor}

\DeclareMathOperator*{\argmax}{arg\,max}

\newtheorem{assumption}{Assumption}

\newtheorem{theorem}{Theorem}

\usepackage{amsmath,amsfonts,bm}

\def\eqref#1{equation~\ref{#1}}

\def\1{\bm{1}}

\DeclareMathAlphabet{\mathsfit}{\encodingdefault}{\sfdefault}{m}{sl}
\SetMathAlphabet{\mathsfit}{bold}{\encodingdefault}{\sfdefault}{bx}{n}

\NewDocumentCommand{\norm}{sm}{\IfBooleanTF{#1}{\|#2\|}{\left\| #2 \right\|}}

\newcommand{\E}{\mathbb{E}}

\title{\LARGE \bf
Robust-Sub-Gaussian Model Predictive Control for Safe Ultrasound-Image-Guided Robotic Spinal Surgery
}

\author{Yunke Ao$^{1,2,3}$, Manish Prajapat $^{2,3}$, Yarden As$^{2,3}$, Yassine Taoudi-Benchekroun$^{4}$, Fabio Carrillo$^{5}$, \\ Hooman Esfandiari$^{1}$, Benjamin F. Grewe$^{4}$, Andreas Krause$^{2,3}$
 and Philipp Fürnstahl$^{1}$%
\thanks{$^{1}$ Yunke Ao, Hooman Esfandiari and Philipp Fürnstahl are with the ROCS team at Balgrist University Hospital and University of Zurich, Forchstrasse 340, 8008 Zürich, Switzerland
{\tt\small yunke.ao@balgrist.ch}}%
\thanks{$^{2}$Yunke Ao, Manish Prajapat, Yarden As and Andreas Krause are with the Department of Computer Science, ETH Zurich, Universitätstrasse 6, 8092 Zürich, Switzerland}%
\thanks{$^{3}$ Yunke Ao, Manish Prajapat, Yarden As and Andreas Krause are with TH AI Center, ETH Zurich, Andreasstrasse 5, 8092 Zürich, Switzerland}
\thanks{$^{4}$ Yassine Taoudi-Benchekroun and Benjamin Grewe are with the Institute of Neuroinformatics, ETH Zurich, Winterthurerstrasse 190, 8057 Zürich, Switzerland}
\thanks{$^{5}$ Fabio Carrillo is with OR-X Translational Center for Surgery, Balgrist University Hospital, University of Zurich, Zürich, Switzerland }
}%

\begin{document}

\maketitle
\thispagestyle{empty}
\pagestyle{empty}

\begin{abstract}
    Safety-critical control using high-dimensional sensory feedback from optical data (e.g., images, point clouds) poses significant challenges in domains like autonomous driving and robotic surgery. Control can rely on low-dimensional states estimated from high-dimensional data. However, the estimation errors often follow complex, unknown distributions that standard probabilistic models fail to capture, making formal safety guarantees challenging. In this work, we introduce a novel characterization of these general estimation errors using sub-Gaussian noise with bounded mean. We develop a new technique for uncertainty propagation of proposed noise characterization in linear systems, which combines robust set-based methods with the propagation of sub-Gaussian variance proxies. We further develop a Model Predictive Control (MPC) framework that provides closed-loop safety guarantees for linear systems under the proposed noise assumption. We apply this MPC approach in an ultrasound-image-guided robotic spinal surgery pipeline, which contains deep-learning-based semantic segmentation, image-based registration, high-level optimization-based planning, and low-level robotic control. To validate the pipeline, we developed a realistic simulation environment integrating real human anatomy, robot dynamics, efficient ultrasound simulation, as well as in-vivo data of breathing motion and drilling force. Evaluation results in simulation demonstrate the potential of our approach for solving complex image-guided robotic surgery task while ensuring safety.
\end{abstract}

\section{INTRODUCTION}

Robotic-assisted technologies have recently emerged as a promising solution to enhance the accuracy and safety of surgery procedures, especially for spinal surgery \cite{1_abe2011novel, 18_lieberman2006bone, 17_lieberman2020robotic}. 
However, current state-of-the-art robotic-assisted spinal surgery systems heavily depend on markers on the bones for patient registration \cite{7_farber2021robotics, 6_d2019robotic}.
The use of optical markers is limited by line-of-sight constraints and may cause additional invasiveness to the patient anatomy.
To eliminate reliance on markers, researchers have explored integrating intraoperative data such as ultrasound (US) \cite{li2023robot,li2024ultrasound}, fluoroscopy \cite{12_jecklin2022x23d,jecklin2024domain}, and RGB-depth (RGB-D) cameras \cite{19_liebmann2021spinedepth}. 
However, the high dimensionality and noise associated with these data sources necessitate advanced planning and control frameworks to ensure safe and accurate robotic surgery.

\begin{figure}[t]
\centering
\includegraphics[width=0.45\textwidth]{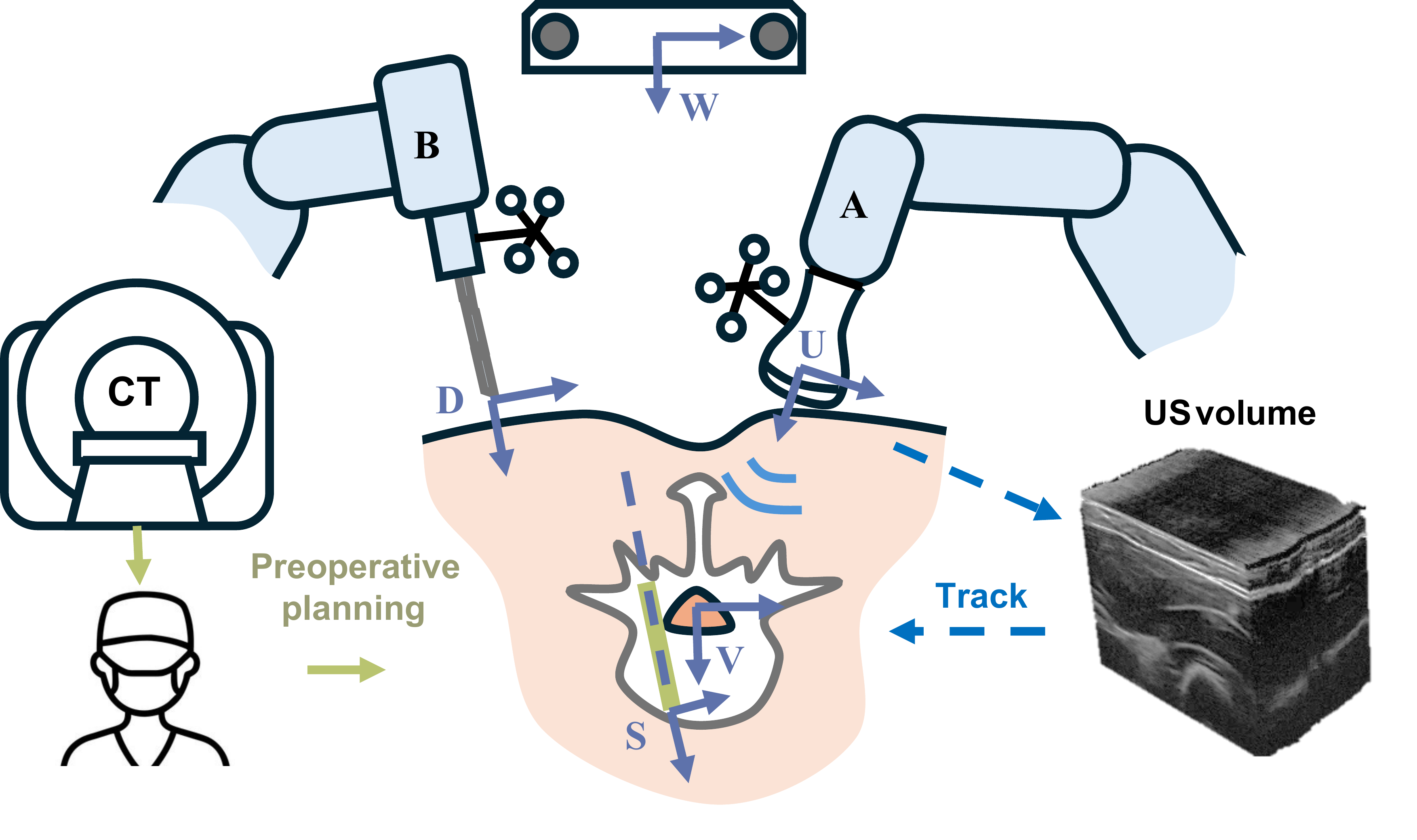}
\caption{Ultrasound (US)-image-guided robotic spinal surgery setup.
Preoperative planning is performed based on the CT scan of the patient, which gives the goal of drilling (yellow green).
The ultrasound, drill, optical camera, and desired placement frames are denoted by $U$, $D$, $W$ and $S$.
}
\label{fig:sys}
\end{figure}

Two notable approaches to address the wider problem of complex control are deep reinforcement learning (DRL) and model predictive control (MPC), each suited to different scenarios based on the needs of the application rather than being direct alternatives.
In particular, DRL has also been applied to safe intraoperative surgical planning for procedures such as Pedicle Screw Placement (PSP), demonstrating higher safety rates than traditional methods in simulations \cite{ao2024saferplan}. 
However, DRL approaches often lack explainability and safety guarantees, which are crucial for applications such as robotic surgery.

In contrast to DRL, MPC methods offer safety guarantees under specific noise conditions \cite{mayne2006robust,hewing2020recursively}, however, limited work addressed safety guarantees for general nonlinear or image observations.
Recently proposed sub-Gaussian stochastic MPC \cite{subgau_mpc} provides closed-loop guarantees for sub-Gaussian noises, which can capture zero-mean state estimation errors originating from image observations.
Nonetheless, real-world intraoperative modalities such as ultrasound can introduce non-zero mean noises for state estimation, presenting additional challenges for control with guaranteed safety.
These constraints have significantly limited the applications of MPC and DRL in real-world image-guided surgery and orthopedics.

In this work, we addressed existing challenges by proposing an MPC framework with \emph{closed-loop guarantees} under general \emph{sub-Gaussian} noise with \emph{bounded mean}, which is applied to US-guided robotic spinal surgery. 
Specifically, our pipeline includes a novel uncertainty propagation method that separately characterizes \emph{biases and variance} of noise, employing robust set propagation techniques for the former and sub-Gaussian approaches for the latter.
To validate our framework, we construct a realistic simulation for US-guided PSP based on Orbit \cite{mittal2023orbit}, incorporating the ITIS human model dataset \cite{gosselin2014development}, efficient US simulation \cite{salehi2015patient}, as well as breathing motion and drilling force data from existing in-vivo experiments conducted on porcine models. 
Evaluation results in this simulation demonstrate the potential of our approach for enhancing safety in robotic surgery and similar vision-based control problems.

\textit{Notation:}
We use $\|x\|_V$ to denote $\sqrt{x^\top Vx}$ for $x\in \mathbb{R}^n$ and $V\in \mathbb{R}^{n\times n}\succeq 0 $. 
We denote the probability of an event $E$ by $\mathrm{Pr}\{E\}$.
We denote the indicator function by $\mathbb{I}(\cdot)$, where $\mathbb{I}(E)=1$ if the event $E$ is true and $0$ otherwise.
We use $\mathbb{E}$ to denote the expectation.
We denote the natural number set by $\mathbb{N}$.
We denote Minkowski sum between sets by $\oplus$.
We use $_A^BT \in \mathrm{SE}(3)$ to denote the rigid transformation from frame $B$ to frame $A$.
We use $\mathcal{K}_\infty$ to denote the set of continuous functions $\alpha: \mathbb{R}_{\geq0} \rightarrow \mathbb{R}_{\geq0}$ which are strictly increasing, unbounded, and satisfy $\alpha({0})=0$.

\section{PROBLEM DESCRIPTION}

\subsection{System Description}
\label{sec:sys}

Our system for ultrasound-guided robotic spinal surgery follows the existing two-robot arm setup in \cite{Li:TMRB:2024} as is shown in \Cref{fig:sys}. 
The surgical procedure begins with a preoperative CT scan, providing the target 3D bone model ($\mathcal{B}$) and enabling surgeons to plan the desired drilling position on the bone model. 
During surgery, our setup (shown in \Cref{fig:sys}) includes an optical tracking camera (world frame, $\{W\}$), a 3D ultrasound probe ($\{U\}$) held by a robot arm (robot A), a surgical drill ($\{D\}$) held by a second robot arm (robot B), the target vertebra ($\{V\}$) on the patient, and the target drilling tip pose ($\{S\}$).
The target drilling pose in the vertebra frame $^V_ST$ is provided by preoperative planning.
Optical tracking gives $_U^WT_t$ and $_D^WT_t$.
The transformation $_V^U T_t$
are not known and are thereby estimated from the US volume images $I$.
We consider using 3D US probe in our setting since 2D US lacks enough freedom to track the real-time 3D pose of the target bone, which can possibily move during the surgery.

We assume that the target vertebra exhibits an \emph{unknown periodic motion due to breathing}, and that the drill may experience \emph{random force disturbances} from the surrounding soft tissues and the bone. 
We thus mount a force sensor on the end-effector of robot B to quantify these external forces $f_e$.
Our system aims to accurately drill along the preplanned path towards $S$ guided by $I$, ensuring precise and safe drilling following the preoperative plan.
The safety requirements (shown in \Cref{fig:model} (a)) include (i) limiting lateral breaches to less than 2 mm outside the pedicle region and (ii) preventing any breakage through the cortical layer into soft tissues.
In the following, we model these requirements analytically and formulate the task space planning problem.

\begin{figure}[t]
\centering
\vspace{0.5em}
\includegraphics[width=0.45\textwidth]{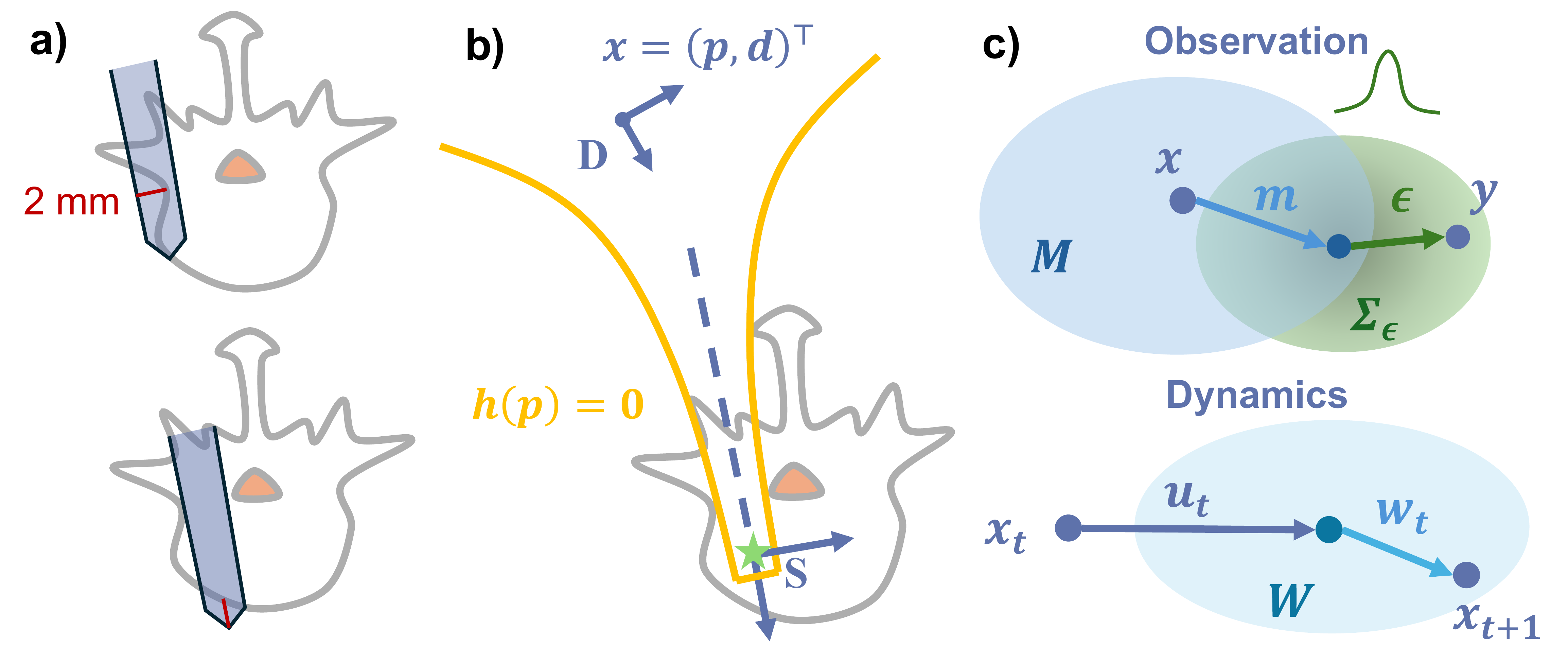}
\caption{Modeling of the task. (a) Lateral breaches (top) and break-through (bottom); (b) simplified constraints; (c)  observation model and dynamics.
}
\label{fig:model}
\end{figure}

\subsection{Task Space Planning Problem Formulation}
\label{sec:model}

\textbf{State} The state $x_t:=[\bm{p}_t; \bm{d}_t]\in \mathbb{R}^5$ represents the pose of the drill tip in the $S$ frame, including position $\bm{p}:=[p_x; p_y; p_z]\in \mathbb{R}^3$ and spherical coordinates of the direction $\bm{d}:=[\theta; \phi]\in \mathbb{R}^2$. 
The state $x_t$ can be expressed by
$x_t = f(_D^ST_t) = f\left({{_V^ST}}\cdot {{_U^VT}_t}\cdot {{_W^UT}_t}\cdot {_D^WT_t}\right)$,
where $f$ is the mapping from the transformation matrix to the pose.
The ground truth state $x_t$ is unknown because ${{_U^VT}_t}$ is unknown.

\textbf{Measurement} We define the measurement $y_t$ as the estimation of $x_t$ from the ultrasound image $I$ (containing information on ${{_U^VT}_t}$):
\begin{align}
    y_t = x_t + m_t + \epsilon_t,~m_t \in \mathcal{M},~\epsilon_t\sim \mathcal{SG}(0, \Sigma_\epsilon), 
    \label{eq:obs}
\end{align}
where $\epsilon_t,m_t$ are bias and random components of the estimation error, as shown in \Cref{fig:model} (c).
We use $\mathcal{SG}(0, \Sigma)$ to denote zero-mean sub-Gaussian distributions with variance proxy $\Sigma\succ 0$ \cite{subgau_mpc}.
The set $\mathcal{M}$ is bounded.
Although both $\epsilon_t,m_t$ can depend on $x_t$, $\Sigma_\epsilon$ and $\mathcal{M}$ can be chosen to be uniformly hold for all $x_t$.
The approach to obtain $y_t$ from $I$ will be detailed in Section~\ref{sec:us_seg} and \ref{sec:image_reg}.

\textbf{Input and dynamics} Our high-level planning input $u$ is defined as $\Delta x$, which represents the end-effector command sent to the low-level controller of robot B for tracking. 
We consider a discrete-time linear dynamics (shown in \Cref{fig:model} (c)) given by:
\begin{subequations}
    \begin{align} 
    & x_0 \sim \mathcal{SG}(\mu_0, \sigma_0), \\
    &x_{t+1} = x_t + u_t + w_t,~w_t\in \mathcal{W},  
\end{align} \label{eq:dyn}
\end{subequations}
where $\mu_0$ and $\sigma_0$ are the prior mean and variance proxy of the initial state.
The disturbance $w_t$ is the tracking error originating from (1) unknown vertebra movement caused by breathing and (2) tracking errors from the low-level robot controller.
The bounded set $\mathcal{W}$ always contains the error $w_t$.

\begin{figure*}[t]
\centering
\vspace{0.5em}
\includegraphics[width=0.9\textwidth,height=0.34\textwidth]{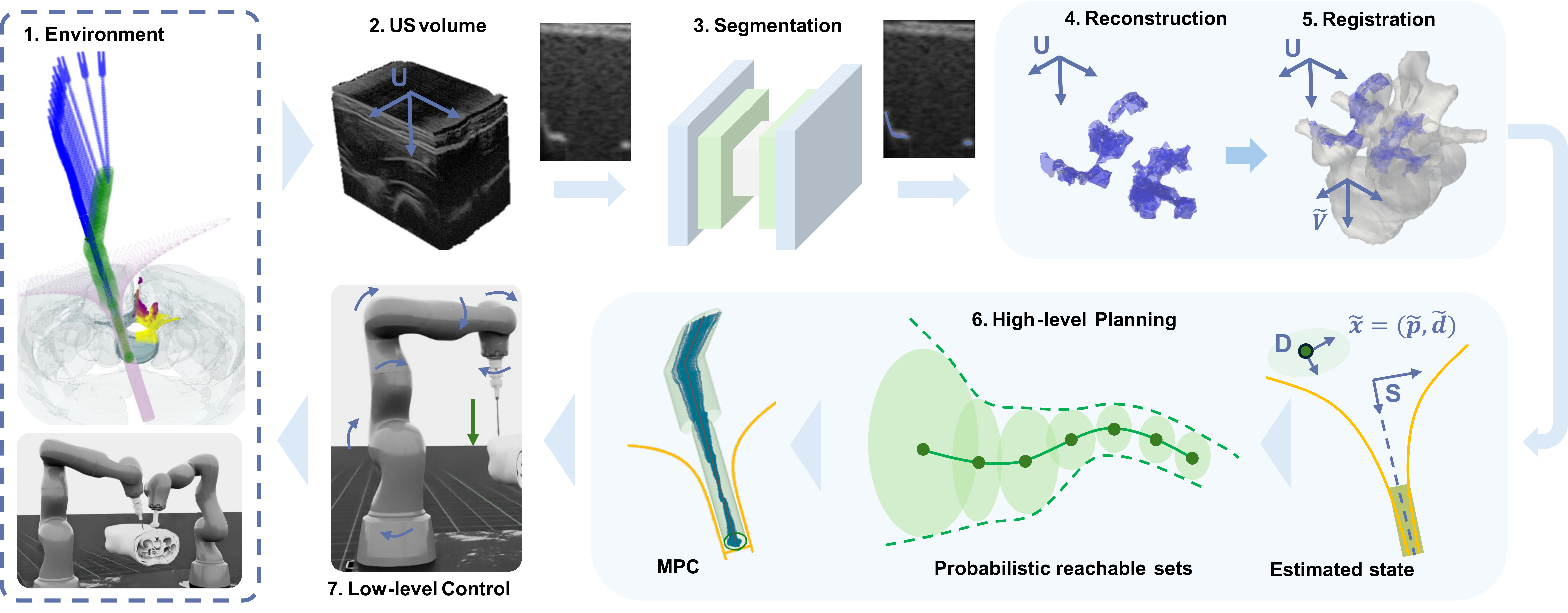}
\caption{Overview of our ultrasound-image-guided planning and control approach.
Real-time ultrasound (US) volume is captured by the transducer and processed into 2D slices.
These US images are segmented by a U-Net to obtain bone surfaces.
We then reconstruct the 3D bone surface from 2D segmentations and register the surface to the CT bone model.
From the registration, we estimate the state, i.e., the relative pose between drill tip and the desired screw placement.
Probabilistic reachable sets (PRS) are then computed for the future states, which are utilized for chance constraint satisfaction with MPC.
The output commands from MPC are finally executed by the low-level controller of the robot that holds the drilling end-effector (robot B).
}
\label{fig:overview}
\end{figure*}

\textbf{Constraints} Our safety requirement (\Cref{sec:sys}) includes lateral breaches of less than 2 mm, which depends on the shape of the bone and is challenging to model analytically.
Therefore, we simplify it as a funnel-like position constraint, which is wide outside but narrow inside the patient body, as illustrated in \Cref{fig:model} (yellow).
This constraint for the drill tip is expressed as: \begin{align} 
    h(\bm{p}) = \left(\sqrt{\frac{p_y^2}{c_y^2} + \frac{p_z^2}{c_z^2}} + c_1\right)^2 - \exp{\left(-\frac{p_x}{c_x} - c_2\right)} \leq 0 \label{ineq:funnel_cons}
\end{align} where $p_x, p_y, p_z$ are coordinates of $\bm{p}$, $c_x$, $c_y$, $c_z$ are scale factors, and $c_1$ and $c_2$ are constants that control the radius and shift of the funnel, respectively.

We add the same position constraint for the point $(\bm{p} + \bm{l})$, which indirectly constrains the drilling direction: 
\begin{align} 
    h(\bm{p} + \bm{l}) \leq 0 \label{ineq:drct_cons}
\end{align}
where $\bm{l}$ is a vector along the drilling direction with the same length as the screw.
We additionally impose a constraint to prevent breakthrough: 
\begin{align} 
    p_x \leq 0 \label{ineq:break_cons}
\end{align}
On the other hand, our input $u_t$ satisfies:
\begin{align}
    \|u_t\| \leq \bar{u} \label{ineq:input_bound}
\end{align}
where $\bar{u}$ is the bound of input.
The final safe sets $\mathcal{X},\mathcal{U}$ for $x_t,u_t$ and the corresponding chance constraints can thus be summarized as: 
\begin{equation}
    \begin{split}
   & \mathcal{X}:=\{x\,\,|\,\,\text{Ineq.}(\ref{ineq:funnel_cons}),~\text{Ineq.}(\ref{ineq:drct_cons}),~\text{Ineq.}(\ref{ineq:break_cons})\}, \\
   & \mathcal{U}:=\{u\,\,|\,\,\text{Ineq.}(\ref{ineq:input_bound})\}, \\
    &P\left([x_t;u_t] \in \mathcal{X}\times\mathcal{U}\right) \geq 1 - \delta,~\forall~t\in \mathbb{N},
\end{split} \label{ineq:cons}
\end{equation}
where $1-\delta$ is the user-specified desired safe probability, $\mathbb{N}$ is the natural number set.

\textbf{Optimization Problem} We consider the following overall stochastic optimal control problem with chance constraints:
\begin{subequations}
    \begin{align}
    & \min_{\pi_{0:\infty}} \lim_{\tau\rightarrow\infty}\frac{1}{\tau}\sum_{t=0}^{\tau}c(x_t, u_t) \label{eq:cost}\\
    s.t. & ~\text{Equ.}(\ref{eq:obs}),~\text{Equ.}(\ref{eq:dyn}),~\text{Equ.}(\ref{ineq:cons}),~\forall t\in \mathbb{N}, \\
    & u_t = \pi_t(y_{0:t}, u_{0:t-1}),~\forall t\in \mathbb{N}, 
\end{align}\label{eq:soc}
\end{subequations}
where $\mu_0$ and $\sigma_x$ are given, $\pi_{0:T-1}$ is the sequence of control laws.
The cost function $c(x_t, u_t)$ is defined as a standard Linear-Quadratic-Regulator (LQR) cost:
\begin{align}
    c(x_t, u_t) := \frac{1}{2}\|x_t\|_Q^2+ \frac{1}{2}\|u_t\|_R^2,
\end{align}
where $Q$ and $R$ are weight matrices.
In Section~\ref{sec:method}, we will introduce a novel approach to solve the optimization problem described by~\ref{eq:soc}.

\section{IMAGE-BASED PLANNING AND CONTROL}
\label{sec:method}

The overview of our image-based MPC framework is shown in \Cref{fig:overview}.
Given a US volume image $I$, we first segment the vertebra surface with a deep neural network (\Cref{sec:us_seg}) .
The state is estimated by registration between the segmented vertebra surface ($^U\mathcal{B}_t$) and the CT bone model $^V\mathcal{B}$ (\Cref{sec:image_reg}).
Then we solve the optimization problem~(\ref{eq:soc}) using a novel MPC approach that handles sub-Gaussian noise with bounded bias (mean) in \Cref{sec:MPC}.

\subsection{Ultrasound Image Segmentation}
\label{sec:us_seg}

We fine-tuned a pre-trained Feature Pyramid Network (FPN) \cite{FPN_kirillov2017unified}, as implemented in \cite{segmentation_github}, with a ResNet encoder backbone \cite{resnet}, to segment the bone surface ($^U\hat{\mathcal{B}}_t$) from ultrasound images. To train the model, we synthetically generated 800 bone ultrasound images, each accompanied by a binary mask delineating the bone surface from the rest of the bone structure and background (bone surface labeled as 1, background and other structures as 0). We applied standard image augmentation techniques (rotation, mirroring, Gaussian noise, and rescaling) to 600 of these images, producing a total of 9,600 images, of which 10\% were set aside for validation. The remaining 200 images were reserved for testing and were not augmented. After fine-tuning, the model achieved an Intersection over Union (IoU) of $0.52$ on the test set.

\subsection{Pose Estimation with Image-based Registration}
\label{sec:image_reg}
We estimate the transformation $_U^VT_t$ by aligning the CT bone model $^V\mathcal{B}$ to the segmented bone surface $^U\hat{\mathcal{B}}_t$.
Specifically, we search for the transformation $_U^VT^*_t$ that maximizes the volume of intersection between the transformed bone model and the segmented bone surface:
$_U^VT^*_t = \argmax_{T} |^U\hat{\mathcal{B}}_t \cap (T\cdot{^V\mathcal{B}})| \label{eq:reg}$,
where surface models $^V\mathcal{B}_t$ and $^U\hat{\mathcal{B}}_t$ are regarded as sets in 3D space, $T \cdot{^V\mathcal{B}}$ represent transformed $^V\mathcal{B}_t$ by $T$, and $|\cdot|$ denotes the volume of the set.
Given $_U^VT^*_t$, the resulting estimated state $y_t$ can be computed by
\begin{align}
    y_t = f(_D^S\hat{T}_t)=f(_V^ST \cdot {_U^VT^*_t} \cdot {_U^WT_t} \cdot 
 {^W_DT_t}),
    \label{eq:per-step-se}
\end{align}
as is mentioned in \Cref{sec:model} (\textbf{State}).
We assume that the estimation error $y_t - x_t$ has bounded bias and sub-Gaussian randomness, which yields the measurement model \labelcref{eq:obs}.

\subsection{Model Predictive Control with Closed-loop Guarantees}
\label{sec:MPC}

In this section, we propose an MPC framework that explicitly accounts for image registration errors, modeled as sub-Gaussian noise with a bounded mean.
We first design a closed-loop linear control system following the indirect observation feedback MPC approach~\cite{hewing2020recursively}.
We then construct ellipsoid probabilistic reachable sets (PRS) under the non-zero sub-Gaussian noise.
By ensuring all states in PRS satisfy the constraints $\mathcal{X}\times \mathcal{U}$, our MPC provides closed-loop guarantees for the chance constraints~(\ref{ineq:cons}).

\textbf{Closed-loop system} 
Similar to \cite{hewing2020recursively}, we consider nominal and estimated states $z_t, \hat{x}_t$ propagated as:
\begin{subequations}
    \begin{align}
    z_{t+1} &= z_t + v_t \label{eq:nominal}\\
    \hat{x}_{t+1} &= \hat{x}_{t} + u_t + L(y_{t+1} - \hat{x}_{t} - u_t)
    \label{eq:kalman} \\
     u_t &= K(\hat{x}_t - z_t) + v_t,
\end{align} \label{eq:est_track}
\end{subequations} where $L,K$ are the Kalman gain and feedback designed offline with LQG or user-specified values.
The initial nominal state is $z_0:=\mu_0$ and $v_t$ is the nominal input.

\textbf{Biased sub-Gaussian error propagation} We define the estimation error as $\hat{e}_t:=\hat{x}_t - x_t$, and the tracking error as $\bar{e}_t:=x_t - z_t$.
Then according to \labelcref{eq:obs}, \labelcref{eq:dyn}, and \labelcref{eq:est_track}, the total error $e_t:=[\hat{e}_t; \bar{e}_t]$ propagates as:
 \begin{align}
     e_{t+1} = A^e e_t + B^e_1 w_t + B^e_2m_{t+1} + B^e_3\epsilon_{t+1},\label{eq:err_dyn}
\end{align}
\begin{align*}
    A^e:= &\begin{bmatrix}
        I - LC & 0 \\
        - K & I + K
    \end{bmatrix}, \\
    B^e_1:= &\begin{bmatrix}
        I - LC \\
        I
    \end{bmatrix},
    B^e_2 =B^e_3:= \begin{bmatrix}
         -L\\
        0
    \end{bmatrix},
\end{align*}
 with properly designed $K,L$ to ensure $A^e$ Schur stable. 
 We divide $e_t$ into bias ($e_t^b$) and stochastic ($e_t^s$) components, then we can decompose \labelcref{eq:err_dyn} into:
 \begin{subequations}
     \begin{align}
        e_{t+1}^s &= A^e e_t^s + B^e_3\epsilon_{t+1}, \label{eq:err_s}\\
         e_{t+1}^b &= A^e e_t^b + B^e_1 w_t + B^e_2m_{t+1}, \label{eq:err_r}\\
         e_{t+1} &= e_{t+1}^b + e_{t+1}^s,
         \label{eq:err_sep}
 \end{align}
 \end{subequations}
where \labelcref{eq:err_r} only contains bounded terms $w_t,m_{t+1}$, and \labelcref{eq:err_s} only has sub-Gaussian term $\epsilon_{t+1}$.

 \textbf{Uncertainty propagation} We now consider the uncertainty propagation for bias and stochastic terms respectively, as shown in \Cref{fig:theory} (left).
 According to \cite[Thm. 1]{subgau_mpc}, the variance proxy of $e_{t}^s$ (denoted as $\Sigma_t$) propagates similar to covariance propagation. Then from \labelcref{eq:err_s} we have:
 \begin{align*}
     \Sigma_{t+1} = A^e\Sigma_t{A^e}^\top + \sigma_\epsilon^2 B^e_3{B^e_3}^\top,
 \end{align*}
 where $\Sigma_{0}=\sigma_0^2 I$.
 On the other hand, a bound set $\mathcal{F}_t^b$ guaranteeing $e^b_t\in \mathcal{F}^b_{t}$ can be propagated based on \labelcref{eq:err_r}:
 \begin{align*}
     \mathcal{F}^b_{t+1} = A^e \mathcal{F}^b_{t} \oplus B^e_1 \mathcal{W} \oplus B^e_2 \mathcal{M},
 \end{align*}
 where $\mathcal{F}^b_{0}:=\emptyset$.

 \begin{figure}[t]
\centering
\vspace{0.99em}
\includegraphics[width=0.45\textwidth]{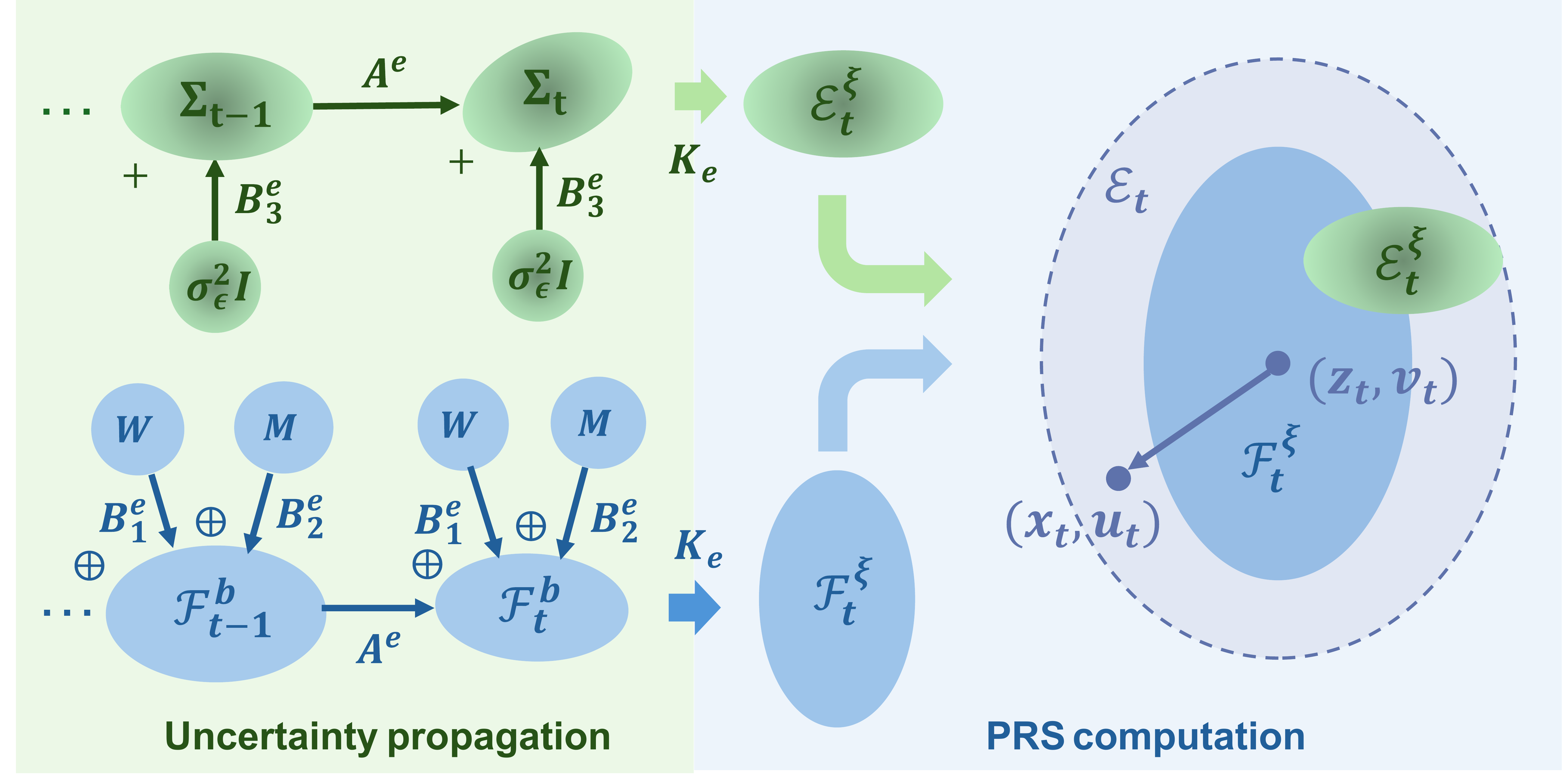}
\caption{Uncertainty propagation (left) and probabilistic reachable sets (PRS) computation (right).
Green and blue ellipsoids illustrate the propagation of the stochastic and bias components of the tracking errors, respectively.
}
\label{fig:theory}
\end{figure}

\textbf{Probabilistic reachable sets (PRS)} 
PRS are sets that include the true states and inputs with greater probability than the user-specified value $1-\delta$.
To compute PRS, we define a stacked error variable $$\xi_t:=\begin{bmatrix}
     x_t - z_t \\
     u_t-v_t
 \end{bmatrix}=\begin{bmatrix}
 0 & I \\
 K & K
 \end{bmatrix}e_t=:K_e e_t,$$ with its bias and stochastic terms denoted as $\xi_t^b:=K_e e^b_t$ and $\xi_t^s:=K_ee^s_t$ respectively.
 Then according to \cite[Thm. 1,2]{subgau_mpc}, a confidence set ($\mathcal{E}^\xi_t$) with $\mathrm{Pr}\{\xi_t^s\in \mathcal{E}_t^\xi\}\geq 1-\delta$ for the sub-Gaussian variable $\xi_t^s$ can be computed from $\Sigma_t$:
 \begin{align*}
     \mathcal{E}^\xi_t:=\{\xi ~|~ \|\xi\|^2_{(K_e\Sigma_t K_e^\top)^{-1} } \leq n_c + n_c \kappa^{-1}(\delta^{-\frac{2}{n_c}})\}, 
 \end{align*} 
 where $\kappa(x):=\frac{e^x}{1+x}$, $n_c$ is the number of dimension of $\mathcal{X}\times \mathcal{U}$. 
 On the other hand, a bounded set ($\mathcal{F}^\xi_t$) for the bias term $\xi_t^b$ with $\xi_t^b\in\mathcal{F}^\xi_t$ can be obtained as $\mathcal{F}^\xi_t := K_e\mathcal{F}^b_{t}$.
 Then by denoting the total error confidence set
 $\mathcal{E}_t := \mathcal{E}^\xi_{t} \oplus \mathcal{F}^\xi_t$, we have $\mathrm{Pr}\{\xi_t \in \mathcal{E}_t\} \geq 1 - \delta$ and thus $\mathrm{Pr}\{[x_t;u_t] \in [z_t;v_t]\oplus\mathcal{E}_t\} \geq 1 - \delta$.
Therefore, $[z_t;v_t]\oplus\mathcal{E}_t$ are valid PRS, as shown in \Cref{fig:theory} (right).
The chance constraints \labelcref{ineq:cons} can be guaranteed by
 $
 [z_t; v_t] \in (\mathcal{X} \times \mathcal{U})\ominus \mathcal{E}_t.
 $

\begin{figure*}[t]
\centering
\vspace{0.5em}
\includegraphics[width=0.9\textwidth]{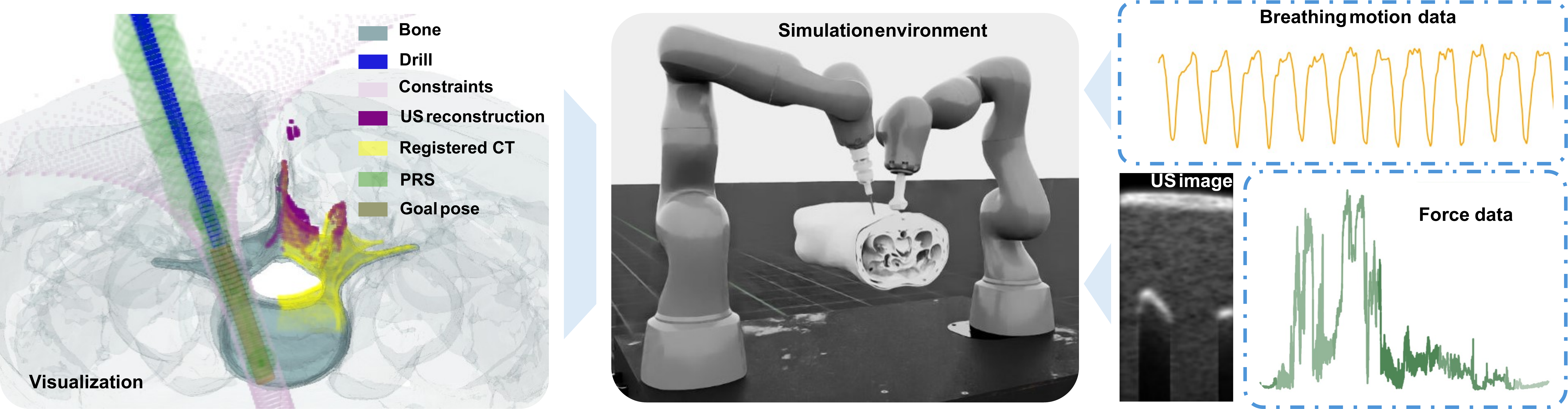}
\caption{Overview of our simulation environment.
The scene in the orbit simulation contains 2 robot arms and a patient human model.
A transparent visualization is shown on the left, with the bone, drill, constraints and probabilistic reachable sets (PRS) colored gray, blue, pink and green respectively.
An example of simulated ultrasound images is shown on the right.
Breathing motion and force data from in vivo experiments are incorporated into the simulation. 
}
\label{fig:sim}
\end{figure*}

\textbf{MPC problem} Following \cite{muntwiler2023lqg}, the constrained optimization problem we solve at each time step is:
\begin{subequations}\label{prob:mpc}
    \begin{align}
        \min_{v_{0:H-1|t}} & c_f(\bar{x}_{H|t}) + \sum_{i=0}^{H-1} c\left(\bar{x}_{i|t}, u_{i|t}\right) \label{eq:mpc_cost}\\ 
        \mathrm{s.t.}~ & \forall \,i\in\{0,...,H-1\}:\\
        & z_{i+1|t} = z_{i|t} + v_{i|t}, \label{eq:z_prop}\\
        & u_{i|t} = v_{i|t} + K(\bar{x}_{i|t} - z_{i|t}), \\
        & \bar{x}_{i+1|t} = \bar{x}_{i|t} + u_{i|t}, \label{eq:x_prop}\\
        &[z_{i|t};v_{i|t}]\in (\mathcal{X}\times\mathcal{U}) \ominus \mathcal{E}_{t+i}, \label{eq:mpc_cons}\\
        & z_{H|t}\in \mathcal{Z}_f, \\
        &\bar{x}_{0|t}=\hat{x}_t,\label{eq:x_init}\\
        &z_{0|t}=z_t,\label{eq:z_init}
    \end{align}
\end{subequations}
where $\bar{x}_{i|t},z_{i|t}$ represent the certainty equivalent states and nominal states at the planned $i$th step from time $t$.
The terminal cost $c_f(x):=\frac{1}{2}\|x\|_P^2$, with $P$ solved by the Lyapunov equation given $A,B,K$.
This problem minimizes the finite horizon cost of the predicted mean trajectory $\bar{x}_{0:H|t},u_{0:H-1|t}$, while guaranteeing the chance constraints satisfaction of the true states $x_{t:t+H}$ through steering the nominal trajectory $z_{0:H|t}$. 
The terminal set $\mathcal{Z}_f$ is designed to satisfy the invariance properties:
\begin{assumption}[Terminal set \cite{subgau_mpc}]
\label{ass:terminal}
    The terminal set $\mathcal{Z}_f$ satisfies for all $z\in \mathcal{Z}_f,t\in\mathbb{N}$: 
    \begin{enumerate}
        \item (Positive invariance) $(A + BK)z\in \mathcal{Z}_f$;
        \item (Constraints satisfaction) $[z,Kz]\in (\mathcal{X}\times \mathcal{U})\ominus \mathcal{E}_t$.
    \end{enumerate}
\end{assumption}

We denote the optimal solution of Problem \labelcref{prob:mpc} as $v^*_{0:H-1|t}$.
The resulting closed-loop system is : 
\begin{align}
    u_t = u^*_{0|t} = K(\hat{x}_t - z_t) + v^*_{0|t},~\text{Equ.}\labelcref{eq:nominal},~\text{Equ.}\labelcref{eq:kalman}\label{eq:ctrl_law}
\end{align}
\begin{theorem}[Closed-loop properties]
    Let \Cref{ass:terminal} hold. %
    Suppose the Problem~\labelcref{prob:mpc} is feasible at $t=0$.
    Then Problem~\labelcref{prob:mpc} is recursively feasible for $t\in \mathbb{N}$ and the closed-loop system \labelcref{eq:obs}, \labelcref{eq:dyn}, \labelcref{eq:ctrl_law} satisfies the chance constraints \labelcref{ineq:cons} for all $t\in \mathbb{N}$. 
    Moreover,  the asymptotic average cost satisfies:
\begin{align*}
    \lim_{T\rightarrow \infty}\frac{1}{T}\sum_{t=0}^{T-1}\E\left[c(x_t, u_t)\right] \leq \kappa_w\left(r_w + r_m\right) + \kappa_\epsilon\left(\sigma_\epsilon\right),
\end{align*}
where $\kappa_w,\kappa_\epsilon$ are $\mathcal{K}_\infty$ functions, $\mathcal{W}$ and $\mathcal{M}$ satisfy $\|w\|\leq r_w,\forall\, w\in \mathcal{W}$ and $\|m\|\leq r_m,\forall\, m\in \mathcal{M}$ respectively. We use $\sigma_\epsilon$ to denote the maximum eigenvalue of $\Sigma_\epsilon$.
\label{thr:mpc}
\end{theorem}
\begin{proof}
Recursive feasibility can be shown by directly substituting the PRS in \cite[Thm. 3]{subgau_mpc} with our PRS $\mathcal{E}_t$.
The proof for asymptotic stability can follow the proof of \cite[Thm. 3]{subgau_mpc}, which justifies the term $\kappa_\epsilon\left(\sigma_\epsilon\right)$ for zero-mean sub-Gaussian noise. The additional terms $\kappa_w\left(r_w+r_m\right)$ introduced by non-zero mean can be derived similarly following the proof of \cite[Thm. 3]{subgau_mpc}.
\end{proof}

\Cref{thr:mpc} states that our controller preserves the \emph{recursive feasibility} and \emph{closed-loop constraints satisfaction}, even though our noise has additional bias terms compared to \cite{subgau_mpc}.

\subsection{Low-level Control}
\label{sec:low_ctrl}
In our setting, the high-level end-effector motions planned by the MPC are executed using a low-level controller of robot B.
We adopt the differential inverse kinematic (DIK) controller and proportional-derivative (PD) controller in Orbit (\cite{mittal2023orbit}) for joint-level control.
Specifically, the DIK controller is first used to compute the desired joint position $q^{des}_t$ by:
\begin{align}
    q_t^{des} = IK\left( _U^W\hat{T}_t\cdot _V^UT_t^*\cdot _S^VT\cdot f^{-1}(\hat{x}_t + k_u u_t)\right)
\end{align}
where $k_u$ is a tunable constant, $IK$ is the inverse kinematic solver that maps $_D^WT_t$ to joint angles of robot B.
We then compute a desired torque given the Jacobian of the robot $J^e$ to compensate for the disturbance force ($\hat{f}^e_t$) measured from the force sensor:
\begin{align*}
    \tau^{des}_t = -{J^e}^\top_t\hat{f}^e_t
\end{align*}
The final actuator torque is solved by the PD controler:
\begin{align*}
    \tau_t^a = k_p(q^{des}_t - q_t) - k_d \dot{q}_t + \tau^{des}_t
\end{align*}
where $k_p$ and $k_d$ are the stiffness and damping of the PD controller, $q_t,\dot{q}_t$ are the joint angles and angular velocities.
The detailed parameters for both high-level and low-level controllers are shown in \Cref{tab:params}.

\begin{table}[t]
\centering
\caption{High-level and low-level controllers parameters.
}
\begin{tabular}{cccccccccc}
\hline
$c_x$ & $c_y$ & $c_z$ & $L$ & $Q$  \\
\cmidrule(lr){1-1}\cmidrule(lr){2-2}\cmidrule(lr){3-3}\cmidrule(lr){4-4}\cmidrule(lr){5-5}
0.01 & 0.2 & 0.2 & 0.99 & $\mathrm{diag}{[100,100,100,10,10]}$ \\ 
\hline
$R$ & $k_u$ & $k_p$ & $k_d$ & $\bar{u}$\\
\cmidrule(lr){1-1}\cmidrule(lr){2-2}\cmidrule(lr){3-3}\cmidrule(lr){4-4}\cmidrule(lr){5-5}
0.1 & 0.9 & 40 & 8 & $\mathrm{diag}{[0.01,0.005,0.005,0.2,0.2]}$\\
\hline
\end{tabular}
\label{tab:params}
\end{table}

\section{EXPERIMENTS}

\subsection{Simulation with Isaac Sim}
\label{sec:sim}
We construct a robotic spinal surgery environment in Orbit simulation platform \cite{mittal2023orbit} based on the system setup described in \Cref{sec:sys}, as is shown in \Cref{fig:sim}.
Both robot arms use the model of KUKA LBR Med 14 R820.
The drill is set with a diameter of $4[\mathrm{mm}]$.
The 3D ultrasound probe is configured similarly to the specifications of the XL14-3 matrix transducer (Philips, Amsterdam, NL), with volume size $50[mm]\times 30[mm] \times 70[mm]$ (width$\times$elevation$\times$depth) and pixel size $0.3[mm]$ (lower resolution for efficiency).
Ultrasound images are simulated in real time using the efficient convolution-based approach detailed in \cite{salehi2015patient}.
We use the ViP3 human model dataset \cite{gosselin2014development} from the ITIS foundation to simulate the patient, which contains segmented anatomies from real MRI data, including vertebra, nerves, and muscle. 

To simulate the realistic breathing motion of the patient, we use the spinal vertical motion data collected in the in-vivo experiment on a porcine model.
We apply the collected vertical movements to the whole human model during the simulation, assuming negligible relative motion between the vertebra and the soft tissue surrounding it.
We also include random disturbance force within a maximum range in the simulation.
The range is configured based on the drilling force data collected in \cite{li2024ultrasound}, with an example of a force curve shown in \Cref{fig:sim}.
Both the force disturbance and breathing motion contribute to $w_t$ in \Cref{eq:dyn}.

\subsection{Experiment Design}
\label{sec:exp_design}
\textbf{Uncertainty Quantification}
Since the set $\mathcal{M},\mathcal{W}$ and $\Sigma_\epsilon$ are not given, we estimate them using simulated data samples.
Specifically, we generate $N=100$ trajectories with length $T$ from the same initial nominal states to the desired screw placement $\{s_{0:T}^i\}_{i=1}^N:=\{x^i_{0:T}, y^i_{0:T},u^i_{0:T}\}_{i=1}^N$, and separate them into 2 groups $\mathcal{G}_1:=\{s^i_{0:T_1}\}_{i=1}^N$ and $\mathcal{G}_2:=\{s^i_{T_1:T}\}_{i=1}^N$, corresponding to states outside the vertebra and states inside the vertebra.
The goal is to estimate the largest $\mathcal{W},\Sigma_\epsilon, \mathcal{M}$ that are valid for all states inside each group.
To do this, we divide each group $\mathcal{G}_j,j=1,2$ into $l$ consecutive segments in time order as $\mathcal{G}_j^1,\mathcal{G}_j^2,...,\mathcal{G}_j^l$.
Under the assumption that within each segment, noise distributions are nearly identical between states, we estimate $\mathcal{M}_j^k,\mathcal{W}_j^k,\Sigma_{\epsilon j}^k$ using the data from each segment $\mathcal{G}_j^k$,
where $\mathcal{M}_j^k,\mathcal{W}_j^k$ are bounds of means of noise samples. 
Estimation of $\hat{\Sigma}_{\epsilon j}^k$ follows \cite{subgau_mpc}, based on the sub-Gaussian definition.
Then the common noise bias bound $\mathcal{M}_j,\mathcal{W}_j$ and variance proxies $\hat{\Sigma}_{\epsilon j}$ for each group are computed by taking the union and maximum over segments, respectively.

\begin{figure*}[t]
\centering
\includegraphics[width=0.99\textwidth]{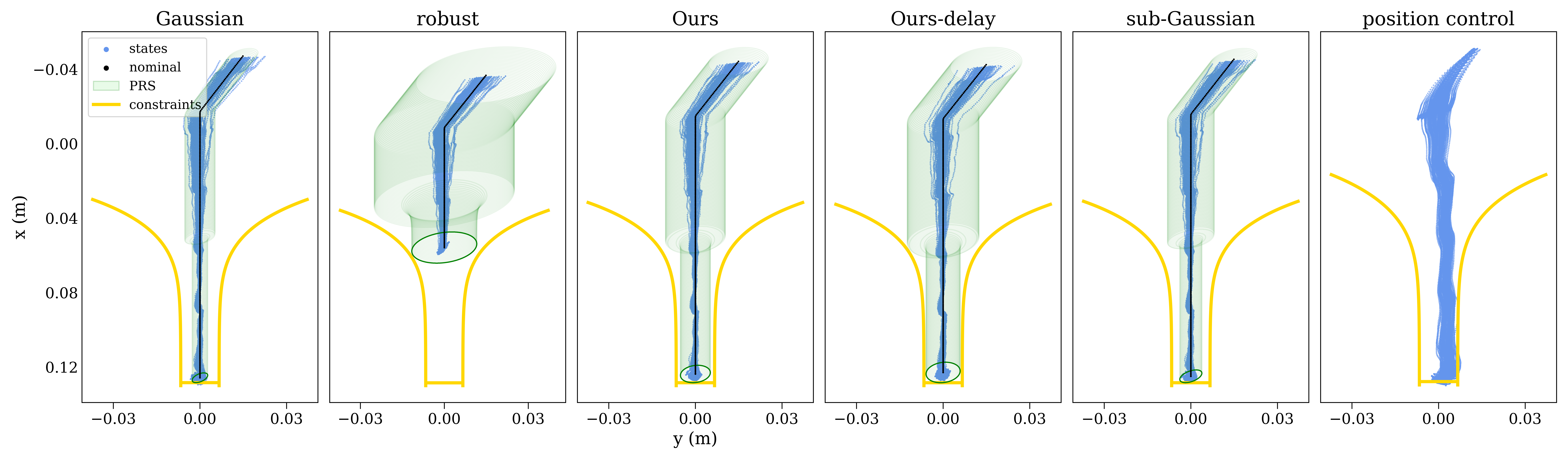}
\caption{Trajectories and PRS in $\{S\}$ frame from different approaches, projected in first 2 dimensions.
We choose $\delta=0.01$.
The true states vibrate around the nominal trajectory in a wave-like form due to the periodic breathing motion.
We use two sets of parameters for noise outside and inside the vertebra respectively, resulting in PRS with 2 levels of sizes, as is mentioned in \Cref{sec:exp_design}.
The robust approach failed to reach the desired screw placement due to too large PRS.
}
\label{fig:mpc}
\end{figure*}

\begin{table*}[t]
\centering
\caption{Evaluation of safety performance of different approaches with $\delta=0.01$ over 100 testing trajectories.
GR A and GR B denote the percentage of no lateral breach and breach of 0-2 [mm], respectively.
The signed distance represents the signed distance of the drill outside the pedicle region.  
MCP and ACP represent minimum and averaged containment probability of the PRS.
We also report the intersection over union (IOU) between the cylinders representing the final drill pose and the desired screw placement. 
Underlined values denote violation of safety constraints or user-specified probability.
}
\begin{tabular}{ccccccccc}
\hline
Methods       & Feasible?  & GR A {[}$\%${]} & GR B {[}$\%${]} & Signed distance {[}$mm${]} & Break ratio {[}$\%${]} & MCP {[}$\%${]} & ACP {[}$\%${]} & IOU [$\%$]\\ \hline
\textbf{Ours} &Y & 33.53 & 66.47 & 0.20 $\pm$ 0.49 & 0.00 & 100.00 & 100.00 & 73.19 \\ 
Ours-delay &Y & 28.95 & 71.05 & 0.26 $\pm$ 0.50 & 0.00 & 100.00 & 100.00 & 71.55 \\ 
sub-Gaussian MPC &Y & 33.96 & 66.04 & 0.19 $\pm$ 0.49 & 0.11 & \underline{84.00} & 99.27 & 74.85 \\ 
Gaussian MPC &Y & 34.29 & 65.71 & 0.19 $\pm$ 0.49 & \underline{4.39} & \underline{23.00} & 91.99 & 76.41 \\
robust MPC &N & - & - & - & 0.00 & 100.00 & 100.00 & - \\ 
position control &Y& \underline{8.93} & \underline{78.91} & 1.22 $\pm$ 0.76 & \underline{6.47} & - & - & 66.50 \\  \hline
\end{tabular}
\label{tab:mpc}
\end{table*}

\textbf{Baselines}
Our approach is compared against robust~\cite{mayne2006robust}, zero-mean sub-Gaussian stochastic~\cite{subgau_mpc}, 
 and Gaussian stochastic MPC~\cite{muntwiler2023lqg} methods.
 Our baseline approaches also include closed-loop classical position control, which are widely adopted by existing works~\cite{li2024ultrasound}.
The position control method directly provides end-effector position commands with constant feeding speed toward the goal position estimated from registration.
These position commands are tracked by the PD controller described in \Cref{sec:low_ctrl}.
We also introduce a variant of our approach with low-frequency (3.3Hz) real-time US volume feedback, to investigate the robustness of our method against signal delay.
This frequency is achievable with existing 3D ultrasound devices like Philips XL14-3.
For all stochastic MPC methods, we set $\delta=0.01$.
All MPC approaches are implemented based on the open source code of~\cite{safeexplore2025}.

\textbf{Metrics} 
In the presented simulation results, each method is evaluated with $N=100$ testing trajectories containing $\{x^i_{0:T},z^i_{0:T},u^i_{0:T},v^i_{0:T}, \mathcal{E}_{0:T}^i\}_{i=1}^N$, where $\mathcal{E}_{0:T}^i$ are total error confidence sets as described in \Cref{sec:MPC}. 
To comprehensively evaluate our PSP pipeline, we employed several metrics that enable assessment of \textit{precision} and \textit{safety}. 
To assess \textit{safety}, we measured (i) the average and maximum containment probability (ACP and MCP) of PRS, (ii) the average signed distance outside the bone, (iii) the break ratio, and (iv), the Gertzbein and Robbins (GR, \cite{gertzbein1990accuracy}) grading system, which is the clinic standard due to its comprehensive evaluation of screw placement \cite{review_PSP_metric_1, review_PSP_metric_2, review_PSP_metric_3}. 
The ACP and MCP are defined as
\begin{align*}
    & ACP:=\frac{1}{N(T+1)}\sum_{i=1}^N\sum_{t=0}^T\mathbb{I}(\begin{bmatrix}
     x^i_t - z^i_t \\
     u^i_t-v^i_t
 \end{bmatrix}\in \mathcal{E}_t^i) \\
& MCP:= \max_{t=0}^T \sum_{i=1}^N \mathbb{I}(\begin{bmatrix}
     x^i_t - z^i_t \\
     u^i_t-v^i_t
 \end{bmatrix}\in \mathcal{E}_t^i)
\end{align*}
The GR system categorizes final screw positions into grades A (fully contained within pedicle) through E (severe misplacement) in increment of 2 mm.
To assess \textit{precision}, we assessed (i) the average position error and its standard deviation, (ii) the average angular error and its standard deviation, and (iii) the cylindrical overlap between the PSP and the ground truth.

\subsection{Results}
\label{sec:exp_results}

\textbf{Uncertainty propagation and constraint satisfaction}
The results of uncertainty propagation and constraint satisfaction are detailed in \Cref{tab:mpc}.
The ACP and MCP of our PRS satisfy the user-specified value, in contrast to the Gaussian and zero-mean sub-Gaussian approaches.
This validates the effectiveness of our uncertainty propagation and quantification methods.
While the MCP of robust approach satisfies the criteria, it fails to find feasible solutions due to the conservatism, as is shown in \Cref{fig:mpc}.
As for clinical metrics, our approach also achieves GR classification better than $B$ and no breaking through for all testing trajectories.
Our GR A ratio is similar to MPC baselines and much better than the classical position control.

\begin{figure}[t]
\centering
\includegraphics[width=0.5\textwidth]{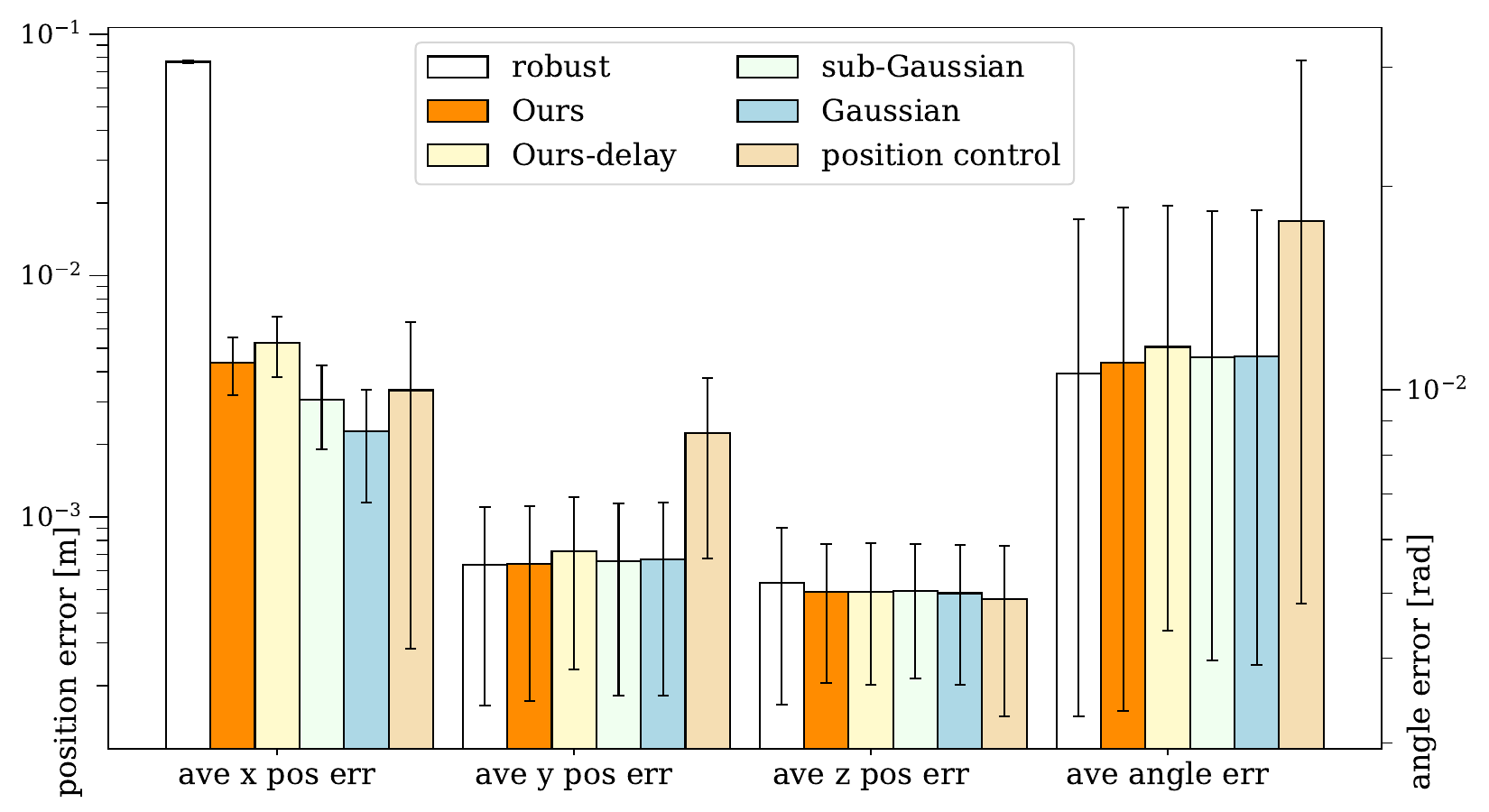}
\caption{Drilling precision of different approaches.
We compare the final pose of the drill $\{D\}$ to the true pose of desired screw placement $\{S\}$.
Specifically, we show averaged absolute position errors along $xyz$ axes of $\{S\}$ frame and the angle between $x$ axes of $\{D\}$ and $\{S\}$ frames.
The standard deviation of each error are also annotated as error bars.
}
\label{fig:precision}
\end{figure}

\textbf{Drilling precision}
The drilling precision results are illustrated in \Cref{fig:precision}.
Our approach provides sub-millimeter lateral position tracking error and less than 1 degree of angle tracking error in our simulation.
The lower frequency of US volume feedback (3.3Hz) results in larger PRS and slightly worse tracking in the drilling direction, but feasible solutions with comparable safe rates are still achievable.
Besides, all stochastic MPC approaches achieve higher precision for horizontal position and angle tracking than classical position control. 
However, our precision in the drilling direction ($x$) is lower due to constraint tightening to prevent drilling from breaking through.
For the same reason, our cylinder overlap ratio (\Cref{tab:mpc}) is also slightly lower than stochastic MPC methods (Gaussian and sub-Gaussian).
In contrast, stochastic MPC and classical positional control methods give smaller tracking errors but higher breaking-through probability, as is shown in \Cref{tab:mpc}.

\section{CONCLUSION}

In this work, we proposed an MPC framework with closed-loop guarantees for chance constraint satisfaction under noise with bounded bias and sub-Gaussian randomness.
We extend the existing sub-Gaussian MPC approach \cite{subgau_mpc} by deriving the uncertainty propagation for biased sub-Gaussian noise.
We developed a simulation environment for ultrasound-guided robotic spinal surgery using Orbit, additionally integrating ultrasound simulation from \cite{salehi2015patient} and ViP3 human-model dataset \cite{gosselin2014development}.
Breathing motion data and drilling force from in-vivo experiments are incorporated to achieve a realistic simulation.
Evaluation results demonstrate the capability of our approach to achieve high clinical performance while satisfying the safety constraints in the simulation.
One limitation of this work is approximating the breathing motion from a porcine experiment, but the data is comparable to humans.
Another limitation is simulating the patient model as a rigid body.
Interesting future directions include the soft tissue deformation simulation, sim-to-real transfer, and collaborated planning for the dual-arm system.

\section{Other Ethics Statements}
Ethics approval has been granted for our study by the Swiss Federal Food Safety and Veterinary Office under Ethical Application Number N°36440  Cantonal N° ZH003/2024.

\section{Acknowledgement}
This work is part of the "Learn to learn safely"project funded by a grant of the Hasler foundation (grant nr: 21039).
We acknowledge ITIS Foundation for providing the virtual population dataset.

\printbibliography

\end{document}